\documentclass[letterpaper, 10 pt, conference]{ieeeconf}  

\IEEEoverridecommandlockouts                              

\usepackage{cite}
\usepackage{graphics}
\usepackage{graphicx} 
\usepackage{amsfonts} 
\usepackage{amsbsy} 
\usepackage{bm} 
\usepackage{amsmath}
\usepackage{algorithmic}
\usepackage[ruled,vlined]{algorithm2e}  
\usepackage{balance}

\title{\LARGE \bf
	Obstacle Avoidance of Resilient UAV Swarm Formation with Active Sensing System in the Dense Environment 
}

\author{Peng Peng, Wei Dong*, Gang Chen and Xiangyang Zhu
	\thanks{All authors are with the State Key Laboratory of Mechanical System and Vibration, School of Mechanical Engineering, Shanghai Jiaotong University, Shanghai, 200240, China (email: \{yc\_pengpeng, dr.dongwei, chg947089399, mexyzhu\}@sjtu.edu.cn).}%
}

\begin{document}

	\maketitle
	\thispagestyle{empty}
	\pagestyle{empty}

	\begin{abstract}
		
		This paper proposes a perception-shared and swarm trajectory global optimal (STGO) algorithm fused UAVs formation motion planning framework aided by an active sensing system. First, the point cloud received by each UAV is fit by the gaussian mixture model (GMM) and transmitted in the swarm. Resampling from the received GMM contributes to a global map, which is used as the foundation for consensus. Second, to improve flight safety, an active sensing system is designed to plan the observation angle of each UAV considering the unknown field, overlap of the field of view (FOV), velocity direction and smoothness of yaw rotation, and this planning problem is solved by the distributed particle swarm optimization (DPSO) algorithm. Last, for the formation motion planning, to ensure obstacle avoidance, the formation structure is allowed for affine transformation and is treated as the soft constraint on the control points of the B-spline. Besides, the STGO is introduced to avoid local minima. The combination of GMM communication and STGO guarantees a safe and strict consensus between UAVs. Tests on different formations in the simulation show that our algorithm can contribute to a strict consensus and has a success rate of at least 80\% for obstacle avoidance in a dense environment.
		Besides, the active sensing system can increase the success rate of obstacle avoidance from 50\% to 100\% in some scenarios.

	\end{abstract}
	
	\section{INTRODUCTION}

	Compared to one single unmanned aerial vehicle (UAV), a swarm of UAVs can effectively increase payload capacity and have a broader sense of the environment. With such merits, swarms could be more efficiently fulfill tasks such as detection, search, mapping, and cooperative transportation \cite{chung2018survey}. 
	In the aforementioned tasks, the swarm is commonly required to maintain a specified formation to traverse a dense field of obstacles safely.
	This requirement brings two challenges.
	First, for a distributed swarm, one main prerequisite is that all UAVs need to reach a safe consensus on trajectory planning and decision-making \cite{nestmeyer2017decentralized} while usually, different UAVs have a different understanding of the environment. 
	Current work \cite{alonso2019distributed} uses the convex hull to express safe space for transmission, which builds a foundation for perception consensus. Nevertheless, perception consensus can not strictly ensure planning consensus due to the local minima existing in the optimization problem.
	Second, for the sake of the dense environment and the overall large volume of UAV formation, it is significant to have a complete perception of the environment to ensure safety \cite{zhou2021raptor}. 
	\begin{figure}[thpb]
		\centering
		\includegraphics[width=3.4in]{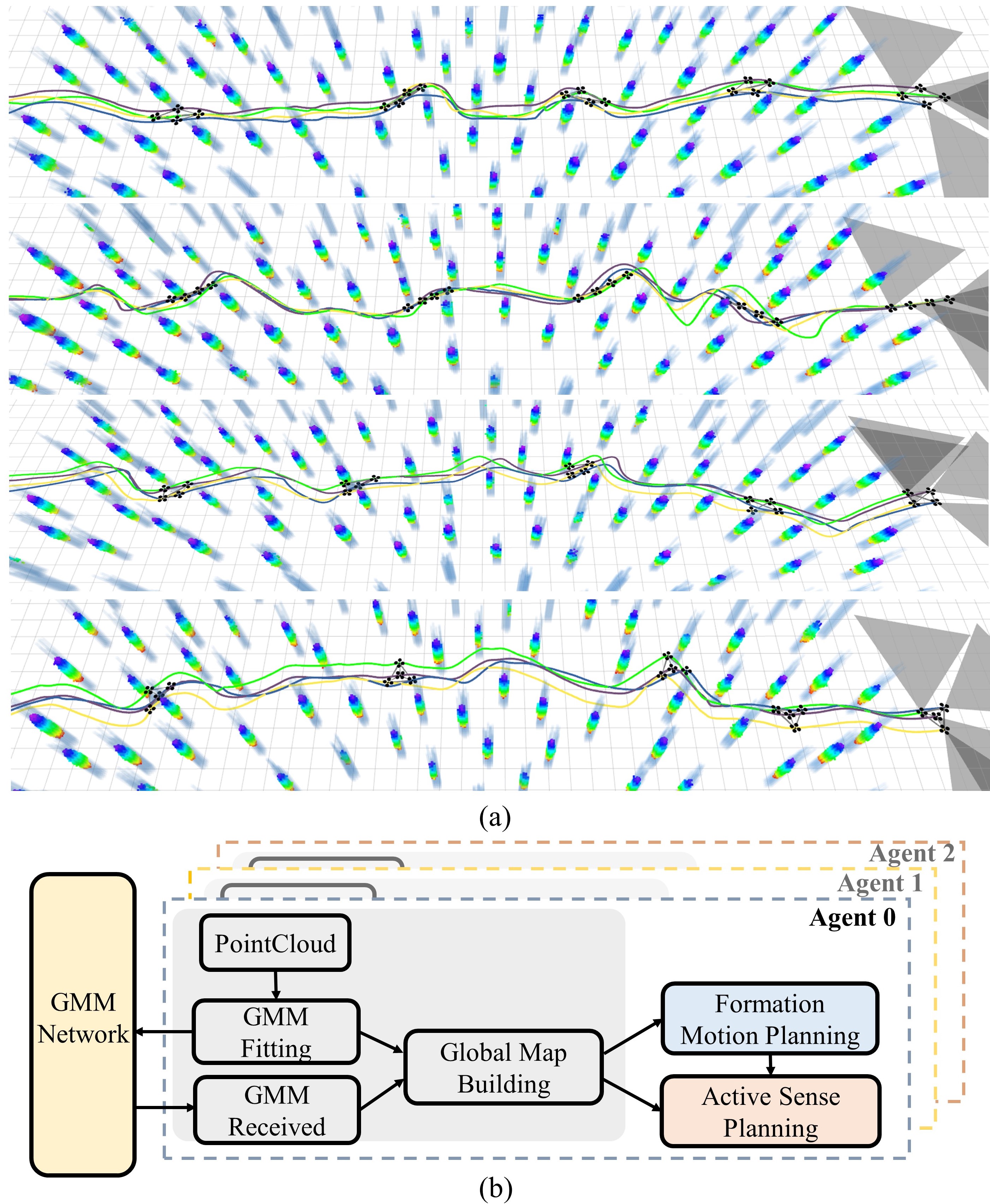}
		\caption{(a) Diamond, straight line, rectangle and triangle formation traverse through the forest scenario from left to right. (b) An overview of our perception and planning system, including formation motion planning and active sense planning.}
		\label{kuang_jia}
	\end{figure}
	Therefore, the FOV of each UAV in the swarm needs to be elaborately planned. Although there has been a lot of work \cite{tordesillas2021panther,chen2021active} considering the field of view of one single UAV, planning FOVs for a UAV swarm in motion planning is still an open problem.

	To bridge the gap of research, we propose a distributed and synchronous motion planning framework for multiple UAVs formation equipped with an active sensing system in the dense obstacle environment, as shown in Fig. \ref{kuang_jia}b. The GMM is used to represent the distribution of obstacles in the environment compactly and transmitted in the swarm. This transmission contributes to a global map for each agent and lays a foundation for the consensus of subsequent planning. The STGO algorithm in the formation motion planning module, together with the GMM communication, guarantees strict consensus for all UAVs. 
	
	In addition, the FOVs of all UAVs are optimized considering the exploration and safety, as well as reducing the overlap of FOVs. This optimization is solved by the DPSO algorithm. 
	
	Besides, in order to improve flight safety, our work allows the affine transformation of the formation structure and treats it as the soft constraint on the control points of the B-spline. The formation trajectory planning is based on the gradient-based method \cite{zhou2019robust}. We adopt the synchronous planning manner and treat the trajectories and observation directions of all UAVs as optimization variables to ensure a feasible solution for all UAVs.
	Simulation is conducted to verify our algorithm for formation maintenance and obstacle avoidance.
	
	We list the main contributions of our work as follows: 
	
	1) A perception-shared and STGO algorithm fused UAVs formation motion planning framework, ensuring a safe and strict consensus on trajectory planning.
	
	2) A distributed active sensing system for the UAV swarm, and to the best of our knowledge, this is the first system considering FOVs in the UAV swarm motion planning task.

	\section{Related Works}
	
	\subsection{Distributed Formation Obstacle Avoidance }
	
	Numbers of methods for UAV formation obstacle avoidance have been proposed. Among them, the lead-follower pattern is widely used due to its convenience of implementation \cite{panagou2014cooperative}. This method highly depends on the leader agent, which makes the formation less robust. Paul et al. in \cite{paul2008modelling} construct an artificial potential field to maintain the formation shape and avoid the obstacle. However, the potential field usually has many local minima, resulting in the robot being trapped.
	
	Recently, motion planning for distributed formation obstacle avoidance has been extensively researched. In \cite{quan2021distributed}, Quan et al. use a graph-based metric to quantify the similarity between two formations. For each agent, the trajectory is optimized based on other agents' pre-planned trajectories. In some scenarios, the pre-planned trajectories may lead to the following UAVs failing in path finding. In other words, the optimized trajectories may not be optimal for the whole swarm.
	Alonso et al. \cite{alonso2019distributed} present a method for formation flight planning based on distributed consensus. The core of this method is to get the same safe convex hull for different robots by intersection, which is taken as the feasible area for planning. This same safe convex hull builds the prerequisite for consensus. Nevertheless, at the same time, this may also make obstacles and trajectories close because it treats the safe space equally \cite{quan2020survey}. Besides, the same convex hull can not strictly guarantee optimization consensus.

	\subsection{Active Sensing System }
	
	A large sensing range could improve flight safety and localization accuracy. Compared to the omnidirectional camera sensor \cite{xu2021omni,schilling2021vision}, active sensing system active vision has advantages in weight and computation consumption. In \cite{zhou2021raptor}, Zhou et al. propose a yaw planning algorithm and define the quality of observation to improve flight safety. In \cite{tordesillas2021panther}, to help the object tracking, the heading of the UAV is planned to keep the dynamic obstacles in the FOV.
	Chen et al. \cite{chen2021active} use a multiple objectives function considering exploration, velocity, and dynamic obstacles to plan for an independent rotation camera with 1 degree of freedom.
	Besides, the active sensing system has also been applied to the UAV swarm. Researchers from \cite{tallamraju2019active} design the optimal trajectory and angular configuration for a UAV swarm to track a target of interest. Zhang et al. \cite{zhang2021agile} use a graph-based active sensing method to improve relative localization for UAVs in the swarm.


	\section{GMM-Based Mapping and Transmission }
	In this work, GMMs are used to compactly represent point cloud and transmitted in the swarm. By resampling from the models, we can reconstruct the obstacles, and build the occupancy grid map and the euclidean signed distance functions (ESDF) map. Compared to standard gaussian distribution, GMM can fit more complex point distributions \cite{corah2019communication} due to the linear combination of multiple gaussian models, which are called components. The probability density function of a GMM with $K$ components is formed as
	\begin{equation}
		p(\boldsymbol{x})=\sum_{k=1}^{K} \pi_{k} \mathcal{N}\left(\boldsymbol{x} \mid \boldsymbol{\mu}_{k}, \boldsymbol{\Sigma}_{k}\right)
	\end{equation}
	where $\mathcal{N}\left(\boldsymbol{x} \mid \boldsymbol{\mu}_{k}, \boldsymbol{\Sigma}_{k}\right)$ is the standard gaussian distribution with mean $\boldsymbol{\mu}_{k}$ and covariance matrix $\boldsymbol{\Sigma}_{k}$. $\pi_{k} $ represents the proportion of each component in the GMMs with $\sum_{k=1}^{K} \pi_{k}=1$. Our work uses the Expectation-Maximum (EM) algorithm \cite{bishop2006pattern} to calculate the optimal $\pi_{k}$, $\boldsymbol{\mu}_{k}$ and $\boldsymbol{\Sigma}_{k}$ of the GMM. EM algorithm solves the maximum-likelihood estimatation in a iterative manner. For the E-step, the probablity of $x_n$ belonging to component $z_k$ is calculated and the probablity is called as the latent variable.
	\begin{equation}
		\gamma\left(z_{n k}\right)=\frac{\pi_{k} \mathcal{N}\left(\boldsymbol{x}_{n} \mid \boldsymbol{\mu}_{k}, \boldsymbol{\Sigma}_{k}\right)}{\sum_{j=1}^{K} \pi_{j} \mathcal{N}\left(\boldsymbol{x}_{n} \mid \boldsymbol{\mu}_{j}, \boldsymbol{\Sigma}_{j}\right)}
		\label{gmm_zk}
	\end{equation}
	Then for the M-step, the optimal parameters are optimized through maximum-likelihood based on the latent variable
	\begin{equation}
		\begin{aligned}
			\boldsymbol{\mu}_{k}^{\text {new }} &=\frac{1}{N_{k}} \sum_{n=1}^{N} \gamma\left(z_{n k}\right) \boldsymbol{x}_{n} \\
			\boldsymbol{\Sigma}_{k}^{\text {new }} &=\frac{1}{N_{k}} \sum_{n=1}^{N} \gamma\left(z_{n k}\right)\left(\boldsymbol{x}_{n}-\boldsymbol{\mu}_{k}^{\text {new }}\right)\left(\boldsymbol{x}_{n}-\boldsymbol{\mu}_{k}^{\text {new }}\right)^{\mathrm{T}} \\
			\pi_{k}^{\text {new }} &=\frac{N_{k}}{N}
		\end{aligned}
		\label{gmm_em}
	\end{equation}
	where $N_{k}=\sum_{n=1}^{N} \gamma\left(z_{n k}\right)$. By iteratively calculating Eq. \ref{gmm_zk} and Eq. \ref{gmm_em}, all parameters can converge to the optimal value. 
	
	One key point of the GMM is that the number of components ${{K}}$ needs to be specified in advance. When faced with an unknown environment, it is difficult for a robot to select a suitable ${{K}}$.
	Commonly, with more components, the GMM can represent the environment more accurately, while this also costs the EM algorithm much more time. Many techniques have been raised to choose an appropriate ${{K}}$ according to the complexity of the environment \cite{corah2019communication}. In our work, ${{K}}$ is proportional to the number of points in the point cloud. This criterion can adapt to most environments and does not require additional computations.
	
	Reconstructing the distribution of obstacles in space is implemented by resampling from the GMM. Meanwhile, an occupancy grid map is built through raycasting from the robot to the sampled points. The state update of each grid in the map follows the sensor measurement model in \cite{books/daglib/0014221}. Moreover, a Euclidean Signed Distance Field (ESDF) map is maintained by Euclidean distance transform (EDT), and both maps are stored in the circular buffer data structure \cite{usenko2017real}. Illustrations for the GMM-based maps are shown in Fig. \ref{gmm}.
	
	Thanks to the highly compact expression of point cloud with GMM, the message for communication only contains the parameters of GMM and the location of the agent. This expression greatly reduces the demand for communication bandwidth from 200 kbps to 5 kbps in simulation.
	\begin{figure}[thpb]
		\centering
		\includegraphics[width=3.4in]{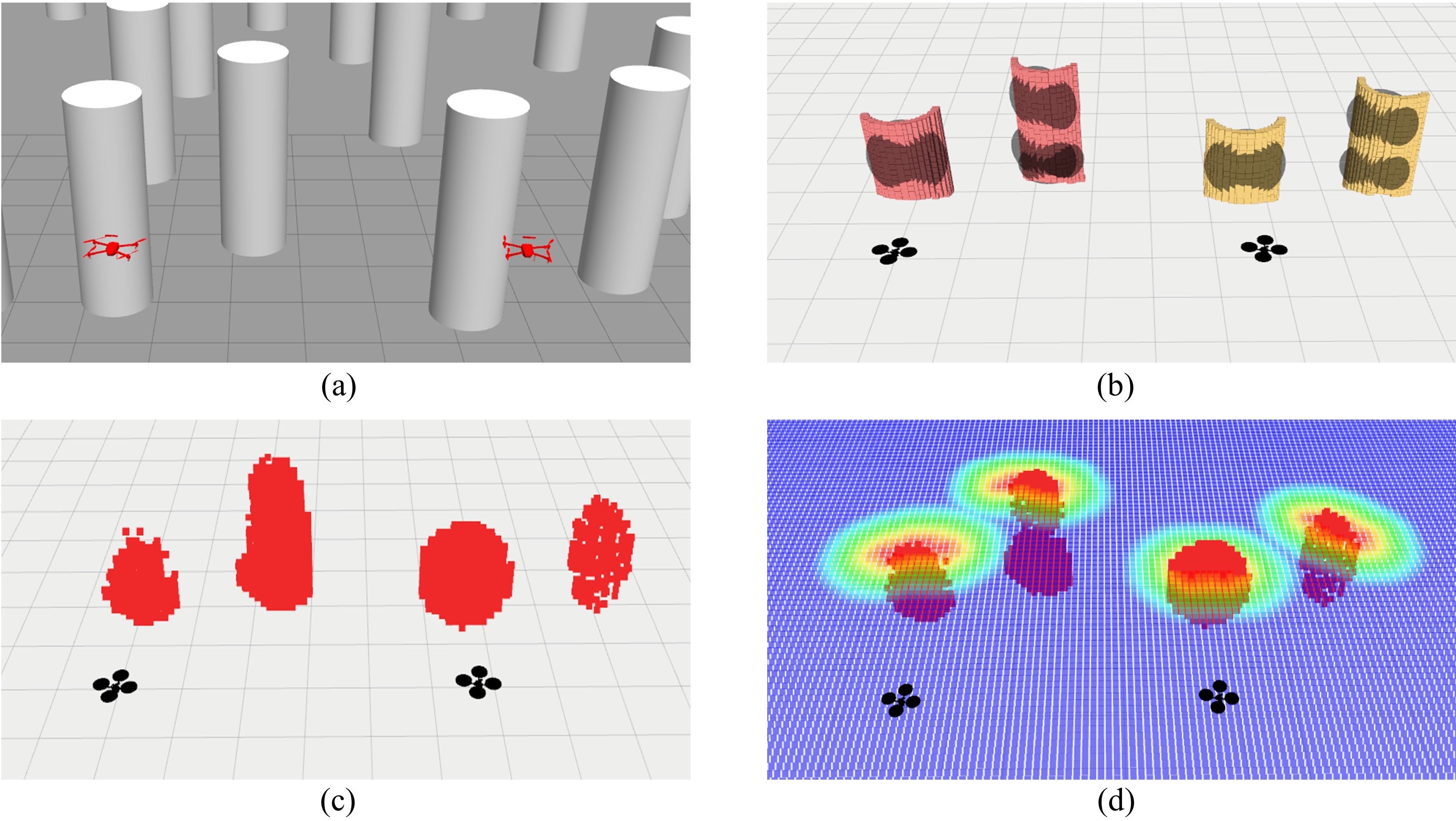}
		\caption{The pipeline for two UAVs building a global map. (a) Two UAVs are equipped with depth cameras in the Gazebo simulation environment. The camera's FOV is limited. (b) The red grids are the point cloud seen by the UAV on the left, and the yellow grids belong to the UAV on the right. The black ellipses are plotted according to the GMM fitting result. Each ellipse represents a component. (c) The global map is constructed by the left UAV after GMM communication and resampling. Through communication, it can reconstruct the obstacles on the right side. (d) A cross-section of the ESDF map built from the global map maintained by the left UAV. 
		}
		\label{gmm}
	\end{figure}
	\section{Active-Sensing System Planning}
	
	When conducting a traversal mission, the UAV swarm is required to have a full understanding of the environment and hence, we adopt the active sensing system for UAVs by planning the yaw angle. For a swarm of $N$ UAVs, denote the angles of all UAV's yaw attitude in the swarm as ${\boldsymbol{\Theta}}=[\theta_0, \theta_1...\theta_{N-1}]^T$, where $\theta_i$ represents the yaw angle of UAV $i$. The positions of UAVs for yaw angles optimization are selected from the pre-planned trajectory described in Sec.V with equal time interval, denoted as $\boldsymbol P(t_{0}), \boldsymbol P(t_{1})... \boldsymbol P(t_{m-1})$. Yaw angles between this selected time point are calculated by linear interpolation.	
	Considering the balance between safety and exploration, the cost function for yaw angles optimization at $t_k$ is defined as
	\begin{equation}
		\begin{gathered}
			\min_{\bm{\Theta_k}} J =\lambda_1 f_{ol}  +\lambda_2 f_{v}-\lambda_3 f_{pv}+
			\lambda_4 f_{s} \\
			\theta_i \in [-\pi,\pi), \quad  i \in [0,N-1]
		\end{gathered}
		\label{head_cost_function}
	\end{equation}
	The $f_{ol}$ function in the first term is defined as $f_{ol}=\sum_{i\neq j}Overlap(\theta_i,\theta_j)$, where $Overlap(\theta_i,\theta_j)$ function
	calculates the overlap area of two FOVs between UAV $i$ and UAV $j$ as shown in Fig. \ref{fov}. In practice, to reduce computational complexity, we project the UAV's FOV onto the same height plane and use the Boost Geometry$\footnote{https://github.com/boostorg/geometry}$ library to get the size of the overlapping regions of the two triangles. Besides, to ensure the safety of the swarm, agents need to pay attention to obstacles in its direction of velocity. Thus, we set $ f_{v}=\sum_i(\theta_i-\theta_i^v)^2$, where $\theta_i^v$ is the velocity direction of UAV $ i$ determined by the trajectory. The third term $f_{pv}= \sum_i PerceptionValue(\theta_i)$ in this cost function evaluates the perception value corresponding to a certain yaw angle. To be specific, in order to have a complete understanding of the whole environment and improve the success rate of path finding, agents tend to observe the field of higher uncertainty. Borrowing from the theory of information entropy, the perception value is defined as $PerceptionValue(\theta_i)=-\sum_{grid_l\in FOV} p_l logp_l $ where $p_l $ indicates the occupancy probability of grid $l$ within the FOV. This function highlights those grids with occupancy probability around $0.5$. 
	The last term in the cost function considers the smoothness of the UAV's yaw rotation and is defined as $ f_s=\left\|\bm{\Theta_k}-\bm{\Theta_{k-1}}\right\|_{2}^2$. Morever, $ \lambda_1$, $ \lambda_2$, $ \lambda_3$ and $ \lambda_4$ determine the weight of each term.
	\begin{figure}[thpb]
		\centering
		\includegraphics[width=2.1in]{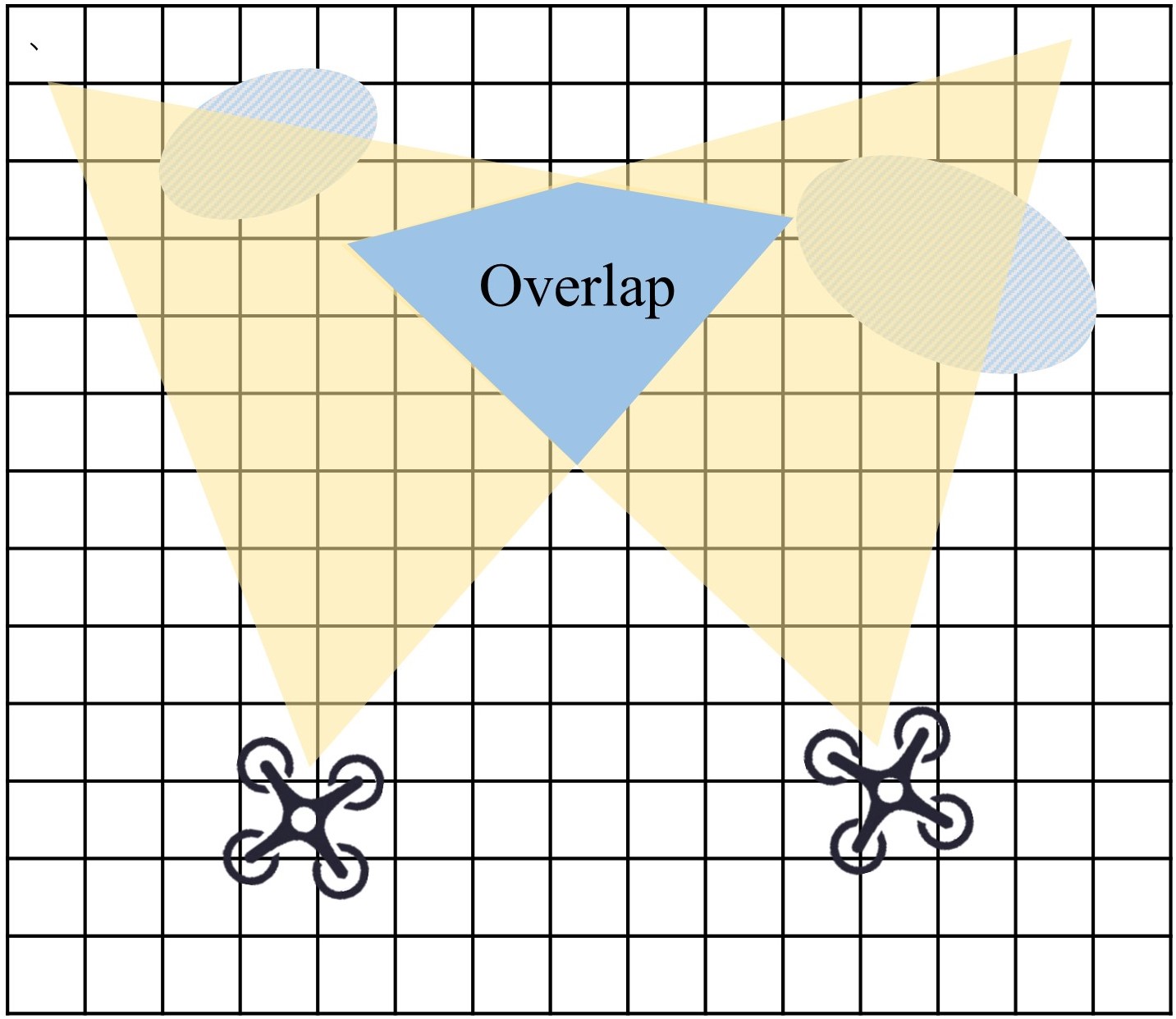}
		\caption{The overlap area of two cameras' FOVs on the same height plane.
		}
		\label{fov}
	\end{figure}
	The optimization problem described above is a typical non-convex problem. In order to take advantage of the computing power of the UAV swarm and avoid local minima, we utilize a DPSO algorithm as shown in Alg. \ref{DPSO}. Compared with standard PSO, DPSO adds the interaction between UAVs and thus increases the number of particles used for computation. Besides, DPSO can guarantee strict consensus between UAVs. Specifically, in the beginning, UAV $i$ initializes a group of particles, and then each particle updates its state concerning the optimal particle in the particle group $P_{gbest}^i$ and its own history optimal value $P_{pbest}^{i,k}$. All UAVs in the swarm broadcast the current optimal particle of their own particle group at a certain frequency. After receiving the information $P_{gbest}^j$, the UAV will compare it with its own group's optimal particle and select the better one. This algorithm achieves optimization consensus in practical and compared with some distributed optimization algorithms that are highly dependent on communication like DGT \cite{xu2015augmented}, the DPSO algorithm can obtain better optimization results in the case of unstable bandwidth.
	\begin{algorithm}
		\label{DPSO}
		\caption{DPSO algorithm}
		\LinesNumbered 
		\KwIn{max iteration times $ M$}
		For agent $i $, initialize a group of particles $S_i$  with random velocities and positions\;    
		\For{episode=1 to  M}{
			\For {particle  $ P_k \in S_i$  } {
				\If{$J(P_k)<J(P_{pbest}^{i,k})$}  {
					$P_{pbest}^{i,k} \leftarrow P_k$\;
					
					\If{$J(P_k)<J(P_{gbest}^{i})$} {
						$P_{gbest}^{i} \leftarrow P_k$\;
					}
				}        
			}
			\For {particle  $ P_k \in S_i$  } {    
				updateVelocityAndPosition($ P_k$)\;
			}

			\If{receive $P_{gbest}^j$ from other agent $j$}  {
				\If{$J(P_{gbest}^j )<J(P_{gbest}^i)$} {
					$P_{gbest}^i \leftarrow P_{gbest}^j$\;
				}           
			}
			broadcast $P_{gbest}^i$ at a certain frequency\;      
		}
		
	\end{algorithm}

	\section{Formation Motion Planning}
	
	The swarm in our work is required to maintain a formation during the traversal task. To avoid obstacle, the formation allows for rotation, scaling and translation. We use a simplified affine transformation to describe formation at time $t$
	\begin{equation}
		\boldsymbol{P}(t)=s \boldsymbol {R}\boldsymbol P_{ref}+\boldsymbol P_{trans}(t)
		\label{affine_transformation}
	\end{equation}            
	where $s$ represents the scale factor, $\boldsymbol R$ is the rotation matrix and $ \boldsymbol P_{trans}(t)$ describes the translation of the formation centre. $\boldsymbol P_{ref}$ is the reference formation structure such as diamond and straight line.
	
	Motion planning for the formation can be divided into two steps. For the first step, the A* algorithm is used to find a path from the formation center to the goal, denoted as $Path_{c}$. Due to the limited sensing range of UAV, when the node in the A* algorithm reaches the map boundary, the algorithm is terminated, and the boundary point is regarded as the endpoint of the path.
	In our work, in order to ensure the connectivity of the swarm during obstacle avoidance, all UAVs are required to bypass the obstacles from the same side. Such requirement is reasonable for some situations, such as vision-based relative localization or cooperative transportation.
	Therefore, the path for each UAV is based on $Path_{c}$ as shown in Fig. \ref{astar}. For each turning point along $Path_{c}$, a new formation is generated around it with a small scale factor for safety, and then each agent's path is found by connecting these new formations. In practice, these paths can be roughly used as the front-end paths for formation motion planning.
	\begin{figure}[thpb]
		\centering
		\includegraphics[width=3.0in]{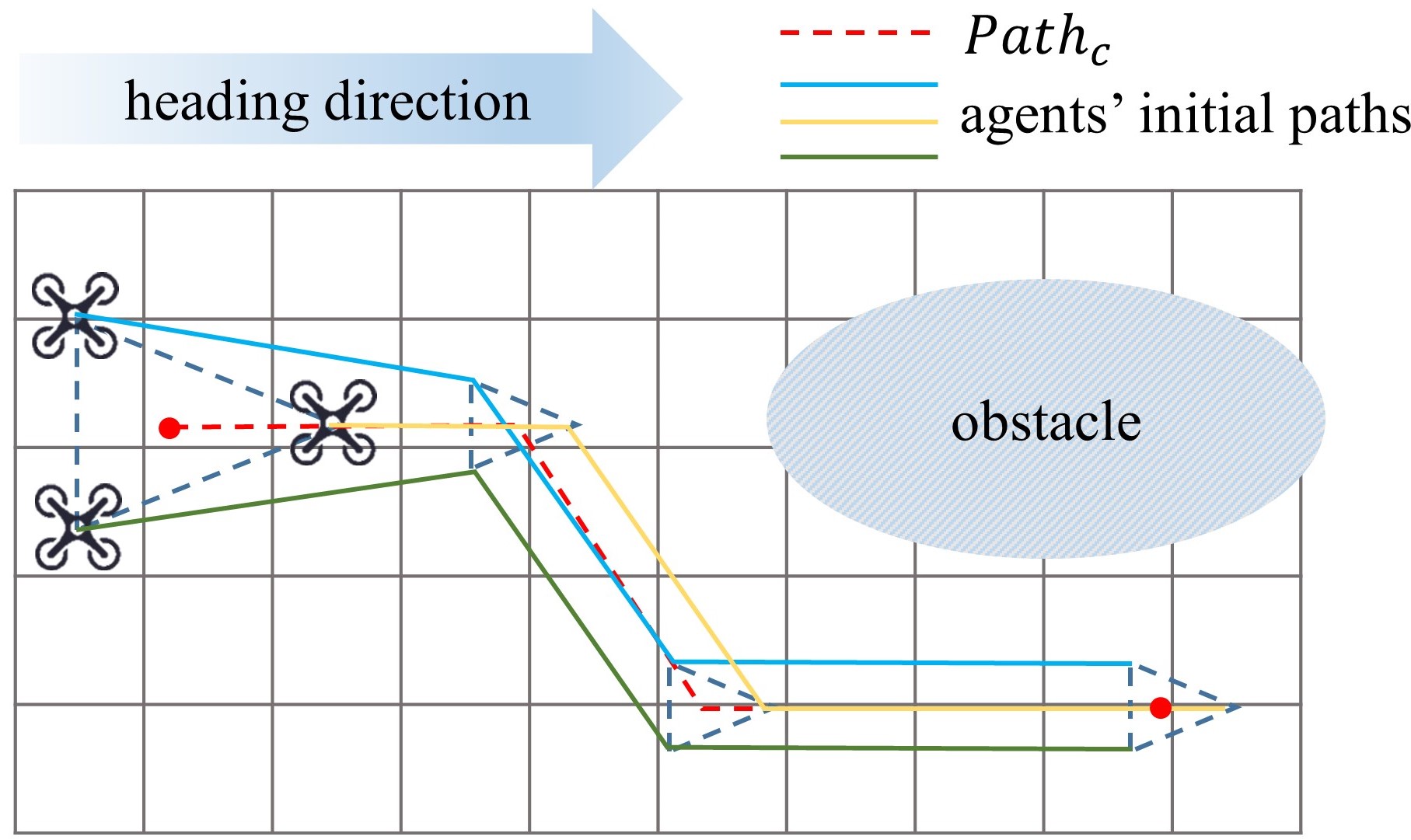}
		\caption{The initial path of each agent. $Path_c$ starts from the formation centre and each agent's initial path is generated based on $Path_c$.
		}
		\label{astar}
	\end{figure}
	
	The second step of motion planning is based on gradient-gased optimization. According to the closed-form solution presented by \cite{richter2016polynomial}, we calculate a polynomial trajectory crossing all turning points of the path found by the first step following the principle of minimum acceleration. Then by sampling on the trajectory, each UAV's trajectory is parameterized as B-spline. Each UAV optimizes the trajectories of all UAVs simultaneously. For a formation with $N$ agents and B-spline curves with $ n$ control points, $ p$ order, the problem is formulated as
	\begin{equation} 
		\begin{aligned}
			&\mathop{\arg\min}\limits_{\boldsymbol  q^i_j, i\in[0,N-1],j\in[3,n-1]} \\&J=
			\lambda_f J_f+\lambda_{b} J_{b}+\lambda_s J_s+\lambda_d J_d+ \lambda_o J_o+\lambda_{r} J_{r} +\lambda_{e} J_{e} 
			\label{trajetory_cost_function}
		\end{aligned}
	\end{equation}
	where $\boldsymbol  q^i_j$  represents the $ j$th control point of agent $i$'s trajectory. $ \lambda_f$, $ \lambda_b$, $ \lambda_s$, $\lambda_d$, $\lambda_o$, $\lambda_{r}$ and $ \lambda_e$ determine the weight of each term.
	Because we specify the initial position, velocity, and acceleration of the trajectory, the $\boldsymbol q^i_0$, $\boldsymbol q^i_1$, and $\boldsymbol q^i_2$ are fixed.
	
	\begin{figure}[thpb]
		\centering
		\includegraphics[width=2.3in]{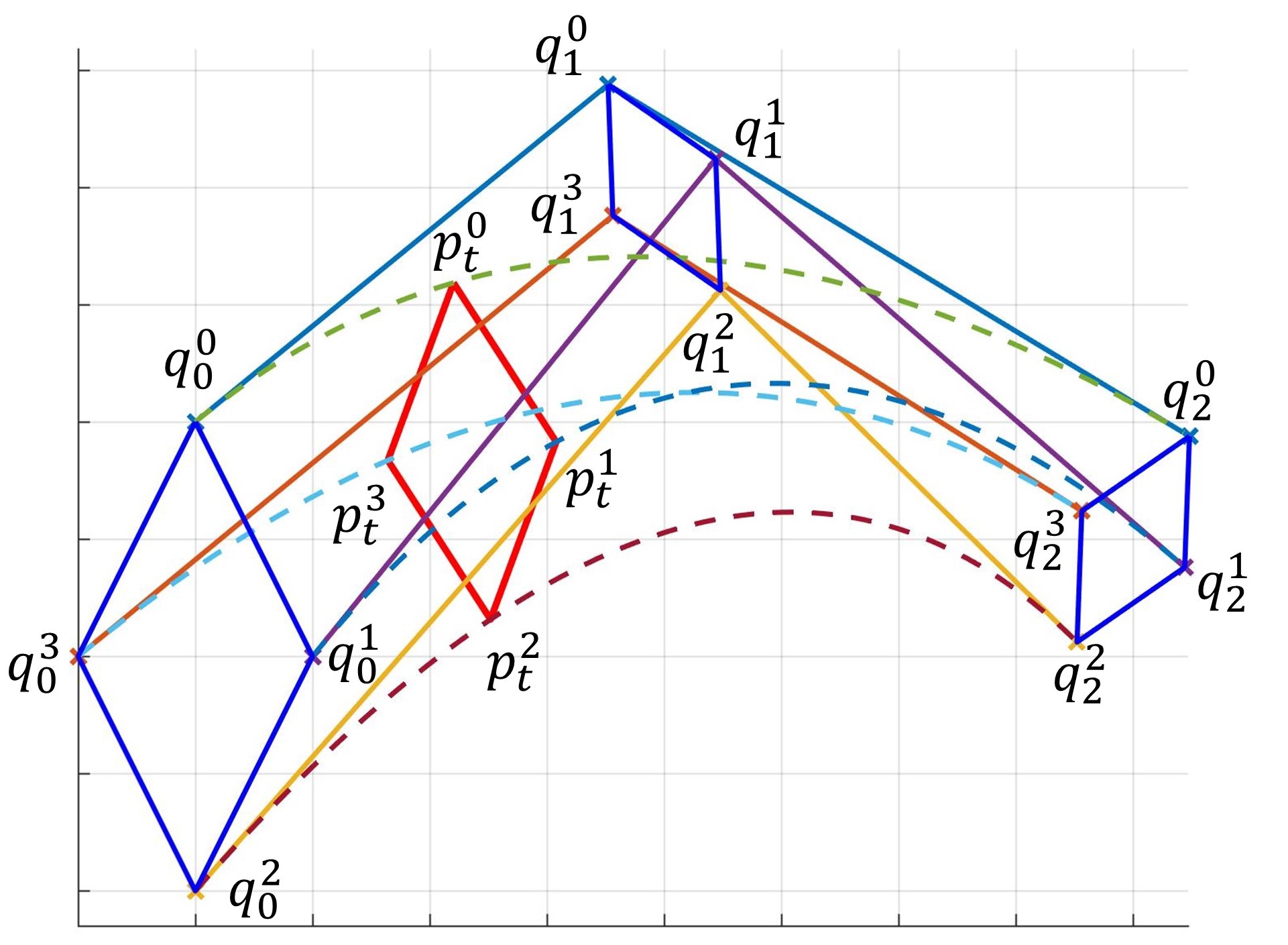}
		\caption{Illustration for Theorem 1. Suppose the reference formation is the diamond. $q^i_j$ is $j$th the control point of UAV $i$'s B-spline. The blue diamonds formed by $q^i_j$ all satisfy the affine transformation relative to the reference formation. According to Theorem 1, at any time $t$, the red diamond formed by UAV $i$'s position $p^i_t$ satisfies the affine transformation relative to the reference formation.
		}
		\label{matlab}
	\end{figure}
	
	To ensure safe flight first, we treat the formation requirement as a soft constraint, and thus, the first term in the cost function measures the degree of formation distortion. The formation constraint is converted to the control points by the following theorem for B-spline curves. Illustration of Theorem 1 is shown in Fig. \ref{matlab}.

	\newtheorem{theorem}{Theorem}
	
	\begin{theorem}
		Suppose the formation formed by control points of all B-splines satisfies the affine transformation compared to the reference formation, then at any time on the trajectory, the formation formed by the agents satisfies the affine transformation compared to the reference formation.
	\end{theorem}
	
	\begin{proof}
		At any time on the trajectory, the formation is determined by control points, $\boldsymbol P(t)=\boldsymbol N(t)\cdot \boldsymbol Q $, where $\boldsymbol P(t) = [\boldsymbol P_0(t),\boldsymbol P_1(t)...\boldsymbol P_{N-1}(t)]$ with $\boldsymbol{P}_i(t) $ denotes the position of agent $i $.
		$\boldsymbol N(t)$ is composed of the base function value which is only related to time $t$. $\boldsymbol{ Q}$ is a matrix of all the control points, $ \boldsymbol Q=[\boldsymbol q_0^T,\boldsymbol q_1^T...\boldsymbol q_p^T]^T $ and $\boldsymbol{q}_j=[\boldsymbol q^0_j, \boldsymbol q^1_j...\boldsymbol q^{N-1}_j]$. $\boldsymbol q_j$ satisfies a general affine transformation compared to the reference formation, $ \boldsymbol q_j=\boldsymbol A_j\cdot \boldsymbol P_{ref}+\boldsymbol B_j$. Then, $ \boldsymbol Q=[\boldsymbol A_0,\boldsymbol A_1...\boldsymbol A_p]^T\cdot\boldsymbol P_{ref}+[\boldsymbol B_0,\boldsymbol B_1...\boldsymbol B_p]^T$ and thus, $ \boldsymbol P(t)=\boldsymbol N(t)\cdot[\boldsymbol A_0,\boldsymbol A_1...\boldsymbol A_p]^T\cdot\boldsymbol P_{ref}+\boldsymbol N(t)\cdot[\boldsymbol B_0,\boldsymbol B_1...\boldsymbol B_p]^T $.
	\end{proof}
	
	Thanks to Theorem 1, we can convert the formation constraint on trajectory points to the control points. 
	According to \cite{zhao2018affine}, the in-plane affine transformation in Eq. \ref{affine_transformation} is determined by three agents, while in our work, the affine transformation is simplified to rotation, scaling and translation, and thus, it can be controlled by two agents, denoted as agent $0$ and agent $1$. In other word, other agents' desired position is a linear combination of agent $0$ and agent $1$ defined as $\boldsymbol{\hat q}^i_j =F (\boldsymbol{q}^0_j,\boldsymbol{q}^1_j)$, where $\boldsymbol{\hat q}^i_j$ is the desired position of the $j$th control point of agent $ i$ satisfying the formation structure. Therefore, the formation cost function is defined as
	\begin{equation}
		J_f=\sum_{i=2}^{N-1} \sum_{j=3}^{n-1}  \left\|\boldsymbol{q}^i_j-\boldsymbol{\hat q}^i_j\right\|_{2}^2
	\end{equation}

	In order to prevent a large degree of scaling factor, we introduce $J_b$ to limit the distance between agent $0$ and agent $1$,
	\begin{equation}
		J_b=\sum_{j=3}^{n-1}(\left\|\boldsymbol{q}^0_j-\boldsymbol{ q}^1_j\right\|_{2}^2-d_{ref}^2)^2
	\end{equation}
	where $d_{ref}$ is the distance between agent $0$ and agent $1$ in the reference formation.
	
	The $ J_s$ in Eq. \ref{trajetory_cost_function} represents the cost function of smoothness. The elastic band cost function \cite{zhou2019robust} is adopted to formulate it.
	\begin{equation}
		J_s=\sum_{i=0}^{N-1}\sum_{j=3}^{n-2} \left\| 2\boldsymbol{q}^i_j-\boldsymbol q^i_{j-1} -\boldsymbol q^i_{j+1}   \right\|_{2}^2
	\end{equation}
	
	Due to the dynamic feasibility, we limit the speed and acceleration at the control points. The cost function to penalize excessive speed and acceleration is defined as
	$ J_d=\sum_{i=0}^{N-1} (\sum_{j=2}^{n-2} g_v(\boldsymbol v_j^i)+ \sum_{j=1}^{n-2} g_a(\boldsymbol a_j^i) )  $ with $g_v(\boldsymbol v_j^i)$ described below and $g_a(\boldsymbol a_j^i)$ follows the same form.
	\begin{equation}
		g_v\left(\boldsymbol v_{j}^i\right)= \begin{cases}\left( ||\boldsymbol v_{j}^i||_2^2-v_{\max }^{2}\right)^{2} & ||\boldsymbol v_{j}^i||_2>v_{\max } \\ 0 & ||\boldsymbol v_{j}^i||_2\leq v_{\max }\end{cases}
	\end{equation}
	where $\boldsymbol{v_j^i}$ is the control point of one order derivative of the B-spline trajectory.	
	
	$J_o$ is used to keep trajectories away from obstacles. It can be quantified by the value in ESDF map. Due to the convex hull property of B-spline, we set the obstacle avoidance constraint on control points, $J_o= \sum_{i=0}^{N-1} \sum_{j=3}^{n-1} h(\boldsymbol q^i_j)$ with  $h(\boldsymbol q^i_j)$ defined as
	\begin{equation}
		h\left(\boldsymbol q^i_j\right)= \begin{cases}\left( E(\boldsymbol q^i_j)-D\right)^{2} & E(\boldsymbol q^i_j) <D \\ 0 & E(\boldsymbol q^i_j) \geq D\end{cases}
	\end{equation}
	where $E(\boldsymbol q^i_j) $ represents the ESDF value and $D$ is the safety distance for obstacle avoidance. The gradient of $E(\boldsymbol q^i_j)$ with respect to $\boldsymbol q^i_j $ is calculated by trilinear interpolation.
	
	$J_r$ is the reciprocal collision cost for formation flight. Theoretically, it is difficult to guarantee that two UAVs will not collide at any time when only the control points are constrained, and thus, we select a series of points on the trajectory at equal time intervals $[t_0,t_1...t_K]$ to apply constraints.
	\begin{equation}
		\begin{gathered}
			J_r=\sum_{i,j= 0,i<j}^{N-1} \sum_{k=0}^Kr(\boldsymbol p^i(t_k),\boldsymbol p^j(t_k))\\
			r(\boldsymbol p^i,\boldsymbol p^j)=\begin{cases}\left( ||\boldsymbol p^i-\boldsymbol p^j ||^2_2-D_r^2\right)^{2} & ||\boldsymbol p^i-\boldsymbol p^j ||_2<D_r\\ 0 & ||\boldsymbol p^i-\boldsymbol p^j||_2\geq D_r\end{cases}
		\end{gathered}
	\end{equation}
	where $\boldsymbol p^i(t_k)$ is the position of agent $i$ at time $t_k$ and $D_r$ is the safety distance for reciprocal avoidance. Besides, $\boldsymbol p^i(t_k)$ is a linear combination of control points, so the gradient with $J_r$ to $\boldsymbol q^i$ can be easily calculated. 
	
	The endpoint of the path found in the first step $ \boldsymbol { \hat p}_{end}^i $ is treated as a soft constraint. Due to the property of the Non-Uniform B-spline, the last control point coincides with the endpoint of the trajectory. Therefore, the cost function for endpoint constraint is defined as
	\begin{equation}
		J_e=\sum_{i=0}^{N-1} \left\|\boldsymbol q^i_{n-1}-\boldsymbol{\hat p}^i_{end}\right\|_2^2
	\end{equation}
	
	Each agent in the swam uses the L-BFGS algorithm from NLopt$\footnote{https://github.com/stevengj/nlopt}$ to optimize the control points of all trajectories. Although all UAVs have a shared map due to the GMM communication in Sev.IV, they can not strictly reach a consensus on trajectories because of the local minima in Eq. \ref{trajetory_cost_function}. Therefore, we introduce the STGO algorithm. After the L-BFGS algorithm converges, every UAV will broadcast its optimization result and the corresponding cost. The result with the lowest cost will be chosen as the final trajectories for all UAVs. In practice, the STGO algorithm effectively avoids local minima and ensures a strict consensus.
	
	\section{Simulation }
	
	We test our active sensing system and formation avoidance algorithm in simulation, and the environment is based on  ego-planner-swarm$\footnote{https://github.com/ZJU-FAST-Lab/ego-planner-swarm}$. The UAVs will perform GMM fitting on the point cloud information provided by the simulation environment and transmit it in the swarm. Similar to the real-world scenario, the camera's FOV is set to 60 degrees, and the maximum distance of the point cloud is 5m. Unlike other works \cite{zhou2021ego,quan2021distributed}, our reference formation tends to bring UAVs together for tasks such as cooperative transportation and relative localization. Typical formations include diamond, straight line, rectangle, and triangle.
	
	\subsection{Active Sensing System}
	
	The DPSO algorithm in our system takes 40 particles on each agent, and the maximum number of iterations is set as 20. The DPSO algorithm takes around 35$ms$, most of which is used to calculate perception value.
	
	We compare our active sensing system with the conventional velocity-direction-consistency
	yaw planning method. As shown in Fig. \ref{head}, for some scenarios like sharp turn, the velocity-direction-consistency yaw planning method can not detect obstacles in time, and this leads to the crash. We conduct 20 tests, and the success rate for the velocity-direction-consistency yaw planning method to cross the corner is 50\%, while with the help of our active sensing system, the success rate is 100\%. Besides, for formations like the straight line, the velocity-direction-consistency yaw planning method causes a large overlap of FOVs as shown in  Fig. \ref{head2}a. Our algorithm can take advantage of the UAV swarm by planning the direction of observation, maintaining full awareness of the environment, as demonstrated in Fig. \ref{head2}b. To be specific, the overlap rate of FOV in Fig. \ref{head2}a is 48.0\% while in Fig. \ref{head2}b, it is 14.7\%.
	\begin{figure}[thpb]
		\centering
		\includegraphics[width=3.4in]{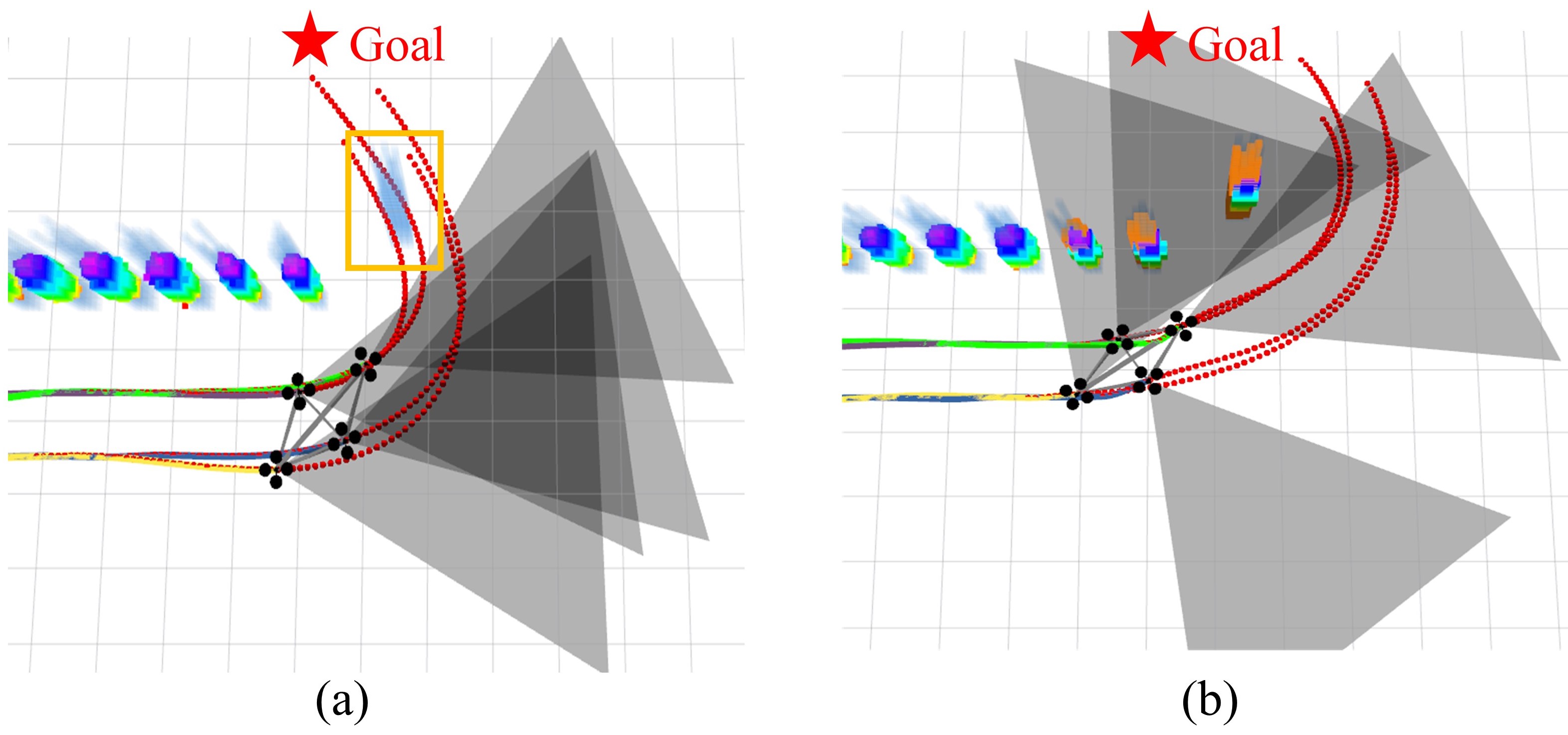}
		\caption{Comparison between two perception planning algorithms when facing sharp turn scenario. Red curves are the planned trajectories, and black transparent triangles are the UAVs' FOVs. Unknown obstacles are in blue. (a) The velocity-direction-consistency yaw planning algorithm cannot detect obstacles marked by the orange rectangle and lead to unsafe trajectories. The success rate to cross the corner is 50\%. (b) Our algorithm enables the UAV to actively observe unknown areas to detect obstacles in time and improve flight safety. The success rate to cross the corner is 100\%	}
		\label{head}
	\end{figure}
	\begin{figure}[thpb]
		\centering
		\includegraphics[width=3.3in]{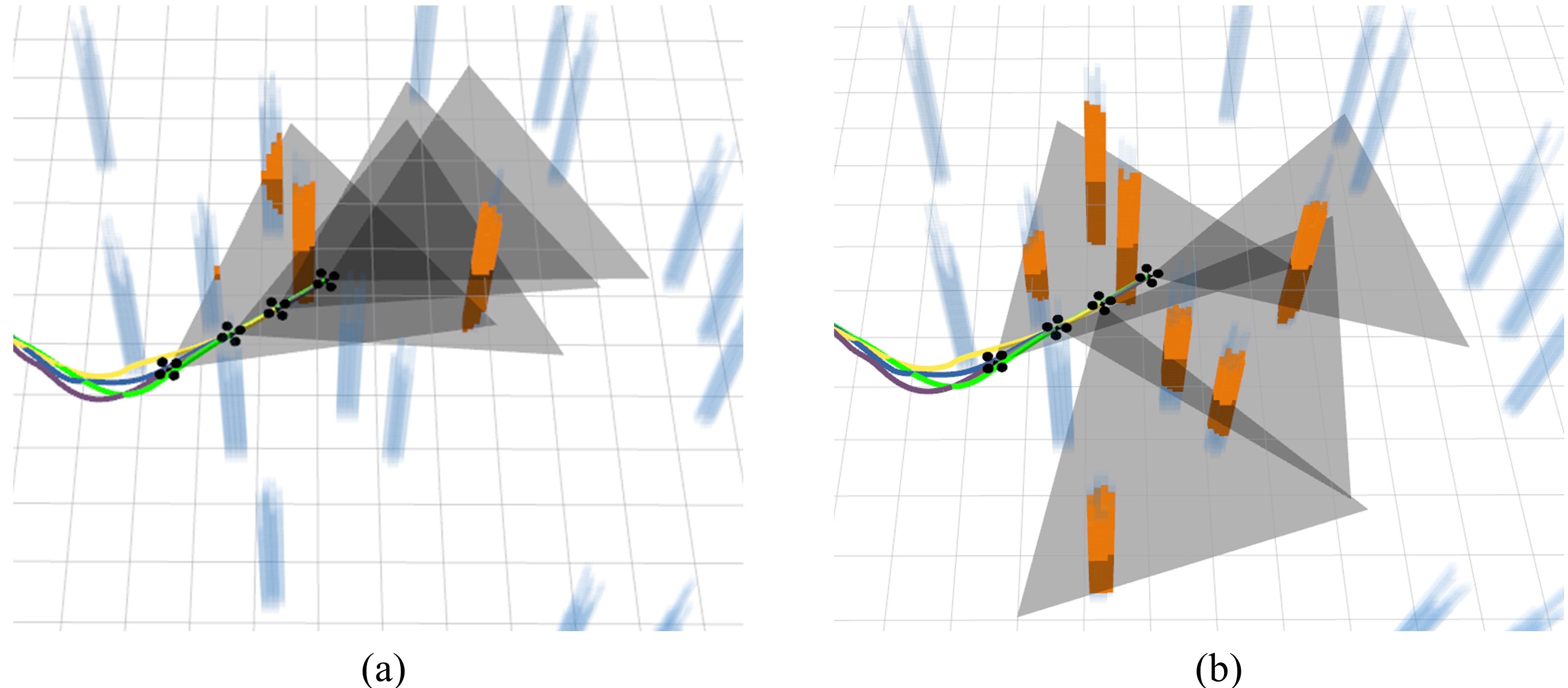}
		\caption{The sense range of two perception planning algorithms. (a) is the velocity-direction-consistency yaw planning algorithm and the overlap rate of FOV is 48.0\%. (b) is our proposed algorithm and the overlap rate of FOV is 14.7\%.	}
		\label{head2}
	\end{figure}
	
	\begin{table}[h]
		\renewcommand\arraystretch{1.5}
		\centering
		\caption{Success rate and $e_{dist}$ with and without GMM or STGO }
		\setlength{\tabcolsep}{4.5mm}
		\begin{tabular}{l|ccc}
			\hline
			Algorithm  & GMM+STGO & GMM & STGO \\
			\hline
			Success rate & 85\% & 10\% & 0\%  \\
			\hline
			$e_{dist}(m)$ & 0.3984 &1.6459 &  0.0725 \\
			\hline
		\end{tabular}
		\label{consensus}
	\end{table}
	
	\begin{table}[h]
		\renewcommand\arraystretch{1.5}
		\centering
		\caption{Success rate and $e_{dist}$ of different formation structures in the forest scenario}
		\begin{tabular}{l|cccc}
			\hline
			Formation structure & Diamond & Line & Rectangle & Triangle \\
			\hline
			Success rate & 85\% & 90\% & 80\% & 80\% \\
			\hline
			$e_{dist}$ & 0.3984  & 0.0234 & 1.0032 & 0.3920 \\
			\hline
		\end{tabular}
		\label{success_rate}
	\end{table}
	
	\subsection{Formation Motion Planning}	
	During the traversal task, all UAVs rely on synchronous timestamps for synchronous planning, which in our system comes from ROS. When the planning timer is triggered, each UAV broadcasts its current position, and then the algorithm described in Sec.V is conducted. We solve the following optimization from \cite{quan2021distributed} to describe the distortion degree of the current formation relative to the reference formation,
	\begin{equation}
		e_{dist}=\min_{\boldsymbol{R},s,\boldsymbol{t}} \sum_{i=0}^{N-1} \left\|\boldsymbol{p}_{i}^{ref}-\left(s \boldsymbol{R} \boldsymbol{p}_{i}^{c}+\boldsymbol{t}\right)\right\|^{2}
	\end{equation}
	where $\boldsymbol{p}_{i}^{ref} $	and $\boldsymbol{p}_{i}^{c}$ is the position of UAV $i$ in the reference and current formation respectively. A smaller $e_{dist}$ represents a smaller distortion.
	
	We demonstrate the success rate and $e_{dist}$ of our system when traversing a forest scenario maintaining a diamond formation with and without GMM communication or STGO in Table \ref{consensus}. As can be seen from the table, without STGO, the distortion degree is large, and there is much reciprocal collision, indicating a bad consensus between UAVs. Without GMM, UAVs can not maintain a full understanding of the environment and thus, evaluate the trajectories cost incorrectly, leading to a $0$\% success rate. Only under the joint help of GMM and STGO ensures a safe and strict consensus between UAVs.
	
	In addition, we use different reference formations to demonstrate our obstacle avoidance performance in an unknown environment. The average speed is about $1.8m/s$, and the motion planning algorithm costs around $30ms$. Fig. \ref{kuang_jia}a illustrates the diamond, straight line, rectangle, and triangle formation in the forest scenario. As the figure shows, all UAVs avoid the obstacles on the same side, and the connectivity of the swarm is maintained during the flight. When encountering narrow corridors, the formation will shrink and temporarily break for safety.
	The success rate and $e_{dist}$ of different formation structures in the forest scenario for 20 tests are shown in Table \ref{success_rate}.

	\section{CONCLUSIONS}
	
	This paper proposes a novel formation perception and motion planning framework. The simulation shows that the GMM and STGO fused method can guarantee a safe and strict consensus in a traversal task with a success rate of more than 80\%. At the same time, GMM can substantially reduce the communication bandwidth from 200kbps to 5kbps. Besides, an active sensing system is embedded in this framework to improve the success rate from 50\% to 100\% in the sharp turn scenario. In the future, we plan to test this framework in the real-world experiment and develop it for a cooperative transportation task.

	\addtolength{\textheight}{-12cm}   
	\bibliographystyle{IEEEtran}
	\balance
	\footnotesize
	\bibliography{IEEEexample}





\end{document}